\documentclass[letterpaper, 10 pt, conference]{ieeeconf}  

\IEEEoverridecommandlockouts                              

\usepackage{amsmath,amssymb,amsfonts}
\usepackage{url}
\usepackage{hyperref}
\usepackage{cleveref}
\usepackage{color}
\usepackage{xspace}
\usepackage{mathtools}
\usepackage{graphicx}      
\usepackage{algorithm}
\usepackage{algorithmic}
\usepackage{wrapfig}

\let\labelindent\relax
\usepackage{enumitem}

\newcommand{\eps}{{\varepsilon}}										
\newcommand{\E}[2][{}]{\mathbb E_{#1}\left[#2\right]}					
\newcommand{\PP}[1]{\mathbb P\left[#1\right]}							
\newcommand{\ra}{\rightarrow}											
\newcommand{\la}{\leftarrow}											

\newcommand{\Z}{{\mathbb{Z}}}											
\newcommand{\R}{{\mathbb{R}}}											
\newcommand{\state}{s}													
\newcommand{\State}{S}													
\newcommand{\states}{\mathbb S}											
\newcommand{\action}{a}													
\newcommand{\Action}{A}													
\newcommand{\actions}{\mathbb A}										
\newcommand{\policy}{\pi}												
\newcommand{\policies}{\Pi}												
\newcommand{\transit}{p}												
\newcommand{\reward}{r}													
\newcommand{\Value}{v}													
\newcommand{\G}{\ensuremath{\mathbb{G}}}								
\newcommand{\K}{\ensuremath{\mathcal{K}}\xspace}						
\newcommand{\KL}{\ensuremath{\mathcal{KL}}\xspace}						

\newtheorem{dfn}{Definition}
\newtheorem{theorem}{Theorem}

\newtheorem{corollary}{Corollary}
\newcommand{\goaldist}[1][{}]{d_{\G#1}}	
\newcommand{\basepolicy}{\policy_{b}}

\newcommand{\altpolicy}{\policy_{a}}
\newcommand{\calfwpolicy}{\policy_{t}}
\newcommand{\Valuebase}{\hat{\Value}^{\basepolicy}}
\newcommand{\Valuealt}{\hat{\Value}^{\altpolicy}}
\newcommand{\subscript}[2]{$#1 _ #2$}

\newcommand{\whiteqedsymbol}{\(\square\)}
{%
	\par\noindent\qquad\textit{Proof#1:}
}%
{%
	\hfill\whiteqedsymbol\par
}

\title{\LARGE \bf
Multi-CALF: A Policy Combination Approach with Statistical Guarantees
}

\author{Georgiy Malaniya$^{1}$, Anton Bolychev$^{1}$, Grigory Yaremenko$^{1}$, Anastasia Krasnaya$^{1}$, Pavel Osinenko$^{1}$ 
\thanks{$^{1}$Skolkovo Institute of Science and Technology}%
\thanks{Corresponding author: P. Osinenko, email: {\tt\scriptsize p.osinenko@gmail.com}.
	First two authors contributed equally.}%
}

\begin{document}

\maketitle
\thispagestyle{empty}
\pagestyle{empty}

\begin{abstract}
	We introduce Multi-CALF, an algorithm that intelligently combines reinforcement learning policies based on their relative value improvements. Our approach integrates a standard RL policy with a theoretically-backed alternative policy, inheriting formal stability guarantees while often achieving better performance than either policy individually. We prove that our combined policy converges to a specified goal set with known probability and provide precise bounds on maximum deviation and convergence time. Empirical validation on control tasks demonstrates enhanced performance while maintaining stability guarantees.
\end{abstract}

\section{INTRODUCTION}

Reinforcement learning (RL) has demonstrated remarkable effectiveness for solving complex control problems across diverse domains, from robotic manipulation \cite{Kumar2016Optimalcontrol,Borno2013TrajectoryOpti,Tassa2012Synthesisstabi,Akkaya2019Solvingrubiks} to games \cite{Silver2016Masteringgame,Silver2018generalreinfor,Vinyals2019Grandmasterlev}. However, standard algorithms like DDPG \cite{Lillicrap2016Continuouscont}, TD3 \cite{Fujimoto2018}, SAC \cite{Haarnoja2018}, and PPO \cite{Schulman2017ProximalPolicy} focus primarily on reward maximization without providing rigorous stability guarantees—a critical limitation for safety-critical applications.

Policy synthesis \cite{Rusu2016PolicyDistillation,Perkins2002LyapunovSafeRL}, which combines policies with complementary strengths, has emerged as a promising approach to address these limitations. These methods range from ensembling \cite{Yang2022PolicyEnsemble} to gating networks \cite{Fernandez2006PolicyReuse} and hierarchical compositions \cite{Sahni2017ComposeSkills}, providing improved performance or safety guarantees through intelligent policy combination.

Our work extends the Critic as Lyapunov Function (CALF) approach \cite{Osinenko2024CriticLyapunov}, which ensures stability through constrained optimization. Our key insight is that policies can be more efficiently combined through a mechanism that selects between them based on their relative critic values. This approach, called Multi-CALF, offers a more flexible selection mechanism than previous Lyapunov-based methods \cite{Perkins2002LyapunovSafeRL}.

Multi-CALF integrates a \emph{base policy} (trained via standard RL to optimize rewards) with an \emph{alternative policy} (equipped with statistical guarantees of reaching a specified goal set). Unlike approaches requiring extensive joint training \cite{Teh2017Distral} or distillation \cite{Rusu2016PolicyDistillation}, Multi-CALF provides a lightweight method for combining independently trained policies at decision time based on relative value improvements.

Our approach seamlessly integrates with existing RL libraries (e.g., stable-baselines3 \cite{stable-baselines3}, cleanrl \cite{huang2022cleanrl}) and speeds up deployment by avoiding additional optimization. Our main contributions are:
\begin{itemize}
	\item A novel policy fusion methodology with stability guarantees for critic-based RL agents
	\item Rigorous theoretical guarantees on goal reaching (Theorems \ref{thm:etagoalreaching} and \ref{thm:unform_goal_reaching})
	\item A lightweight implementation compatible with standard RL libraries
	\item Empirical validation demonstrating improved performance across multiple control tasks
\end{itemize}

Our implementation is available at {\footnotesize\url{https://github.com/aidagroup/multi-calf}}.

\section{PROBLEM STATEMENT}
\label{sec_problem}

In this section, we formalize the reinforcement learning setting and introduce the mathematical framework needed to precisely characterize our goal reaching guarantees.

\subsection{Markov Decision Process}

We work with the standard Markov decision process (MDP) formulation, which provides the foundation for studying the interaction between an agent and its environment:
\begin{equation}
	\label{eq:mdp_definition}
	\mathcal{M}
	\;=\;\bigl(
	\states,\;\actions,\;\transit_0,\;\transit,\;\reward
	\bigr).
\end{equation}

\noindent
Each component of this MDP serves a specific purpose:

\begin{enumerate}[leftmargin=1.5em]
	\item 
	      \(\states\) is the \emph{state space}, representing all possible configurations of the environment.
	      We assume it to be a finite-dimensional Banach space.
	      
	\item 
	      \(\actions\) is the \emph{action space}, representing all possible control inputs.
	      We assume it to be a compact topological space.
	      
	\item 
	      \(\transit_0(\state)\) is the \emph{initial state distribution}—the probability density function of the initial state \(\State_0\).
	      
	\item 
	      \(\transit(\state' \mid \state,\action)\) is the \emph{transition probability density function}.
	      It describes the probability density of transitioning to state \(\state'\) after taking action \(\action\) in state \(\state\).
	      To ensure theoretical guarantees, we assume there exists an upper semi-continuous function \(\bar{p} : \states \times \actions \to \mathbb{R}_{+}\) such that for any \(\state \in \states\) and \(\action \in \actions\), the next state \(\State'\) sampled from \(\transit(\bullet \mid \state, \action)\) satisfies \(\PP{\|\State'\| \leq \bar{p}(\state, \action)} = 1\). This assumption effectively bounds how far the system can move in a single step.
	      
	\item
	      \(\reward : \states \times \actions \;\to\;\mathbb{R}\) is the \emph{reward function}.
	      It quantifies the immediate reward received by taking action \(\action\) in state~\(\state\).
\end{enumerate}

\subsection{Policies and Value Functions}

We distinguish between two types of policies:

\begin{itemize}
    \item \textit{Stationary policies} \(\Pi_{\mathrm{stat}} = \{\pi : \states \to \Delta(\actions)\}\), where \(\Delta(\actions)\) is the set of all probability density functions over \(\actions\). For any state \(\state \in \states\), \(\policy(\action \mid \state)\) represents the probability density of selecting action \(\action\) in state \(\state\).
    
    \item \textit{Non-stationary policies} \(\policies_{\mathrm{nstat}} = \{\pi : \states \times \Z_{\geq0} \to \Delta(\actions)\}\), which explicitly incorporate time dependence. For any \(\state \in \states\) and \(t \in \mathbb{Z}_{\ge0}\), \(\policy(\action \mid \state, t)\) is the probability density of selecting action \(\action\) in state \(\state\) at time \(t\).
\end{itemize}

By definition, \(\Pi_{\mathrm{stat}} \subset \Pi_{\mathrm{nstat}}\), since any stationary policy can be viewed as a special case of a non-stationary policy that remains constant across time steps.

For a discount factor \(\gamma \in [0,1)\) and a policy \(\policy \in \policies_{\mathrm{nstat}}\), we define the \emph{value function} \(\Value^\policy : \states \to \mathbb{R}\) as:
\begin{equation}
	\label{eq:value_function}
	\Value^\policy(\state)
	=
	\E{
		\sum_{t=0}^{\infty}
		\gamma^t\,
		\reward(\State_t,\Action_t)
		\Bigm|
		\State_0 = \state,\,
		\Action_t \sim \policy(\bullet \mid \State_t)
	}.
\end{equation}

This value function represents the expected sum of discounted rewards when following policy \(\policy\) starting from state \(\state\). The standard RL objective is to find a policy that maximizes the expected sum of discounted rewards:
\begin{equation}
	\label{eq:rl_objective}
	\E{\Value^\policy(\State_0)} = \E{\sum_{t = 0}^{\infty}\gamma^t \reward(\State_t, \Action_t)} \ra \max_{\policy\in\policies_{\mathrm{nstat}}},
\end{equation}
where the trajectory \((\State_0,\Action_0,\State_1,\Action_1,\dots)\) is generated by
\(\State_0 \sim \transit_0(\bullet)\),
\(\Action_t \sim \policy(\bullet \mid \State_t, t)\), 
and
\(\State_{t+1} \sim \transit(\bullet \mid \State_t,\Action_t)\) for \(t \ge 0\). 

For a given policy \(\policy \in \policies_{\mathrm{nstat}}\) and an initial state \(\State_0 = \state_0\), we denote the state at time \(t\) by \(\State_t^{\policy}(\state_0)\).

\subsection{Goal Set and Goal Reaching Property}

While traditional RL focuses on reward maximization, many control problems—particularly in safety-critical applications—require driving the system to a specific target region. We formalize this requirement through the concept of a goal set.

\begin{dfn}[Goal Set]
	\label{dfn:goal_set}
	Let \(\G \subset \states\) be a \emph{goal set}—a compact neighborhood of the origin in the state space \(\states\). 
	This represents a target region where we want the system to eventually reach and remain. In practical control problems, this might represent a desired operating point (e.g., an inverted pendulum standing upright) plus a small margin of acceptable deviation.
\end{dfn}

To quantify how close a state \(\state\) is to the goal set \(\G\), we define the distance-to-goal function:
$
	\goaldist(\state)
	:=
	\displaystyle\inf_{\state' \in \G} \,\|\state - \state'\|.
$
This measures the minimum distance from state \(\state\) to any point in the goal set. By definition, \(\goaldist(\state)=0\) if and only if \(\state \in \G\).

Next, we introduce two key properties that characterize policies capable of driving the system to the goal set with high probability. These properties will be central to our stability guarantees.

\begin{dfn}[$\eps$-Improbable Goal Reaching]
	\label{dfn:eta_improbable}
	Let \(\eps \in [0,1)\) be a small probability value.
	A policy \(\policy \in \policies_{\mathrm{nstat}}\) satisfies the $\eps\)-\emph{improbable goal reaching property} if, for all initial states \(\state_0 \in \states\):
	\[
		\label{eqn_introstab}
		\PP{
			\goaldist(\State_t^{\policy}(\state_0))
			\,\xrightarrow{t \to \infty}\,
			0
		}
		\ge
		1 - \eps.
	\]
	
	In other words, the system reaches the goal set in the limit with probability at least \(1-\eps\). This property guarantees asymptotic convergence but doesn't specify how quickly the goal is reached.
\end{dfn}

For more precise control over convergence rate, we introduce a stronger property:

\begin{dfn}[Uniform $\eps$-Improbable Goal Reaching]
	\label{dfn:uniform_eta_improbable}
	Let \(\eps \in [0,1)\).
	A policy \(\policy \in \policies_{\mathrm{nstat}}\) satisfies the \emph{uniform} $\eps\)-\emph{improbable goal reaching property} if there exists a function $\beta \in \KL$ (a class of functions decreasing to zero in both arguments) such that for all initial states \(\state_0 \in \states\):
	\[
		\label{eqn_introstab}
		\PP{
			\goaldist(\State_t^{\policy}(\state_0)) \leq \beta(\goaldist(\state_0), t) \text{ for all } t
		}
		\ge
		1 - \eps.
	\]
	
	This stronger property bounds the distance to the goal at each time step by a function that depends on both the initial distance and elapsed time. It allows us to quantify not just whether the system will reach the goal, but how quickly it will get there and with what maximum deviation.
\end{dfn}

\subsubsection*{Connection to Multi-CALF}

This problem formulation establishes the mathematical foundation for our Multi-CALF approach. The core idea is to combine:

\begin{itemize}
    \item A base policy \(\basepolicy\) that aims to maximize the reward objective \eqref{eq:rl_objective} but may lack formal guarantees
    \item An alternative policy \(\altpolicy\) that satisfies the goal reaching property (Definition \ref{dfn:eta_improbable} or \ref{dfn:uniform_eta_improbable})
\end{itemize}

Our Multi-CALF algorithm, detailed in the next section, adaptively switches between these policies to maintain the theoretical guarantees of \(\altpolicy\) while preserving the performance benefits of \(\basepolicy\).

\section{PROPOSED APPROACH}
\label{sec:approach}

\emph{Multi-CALF} offers a methodology for combining reinforcement learning policies with complementary strengths, inheriting stability guarantees while potentially improving overall performance.

At its core, our approach integrates two distinct policies:
\begin{itemize}
    \item A \emph{base policy} $\basepolicy$: Any policy trained using standard RL methods to maximize rewards. This policy may excel in certain regions of the state space but lacks formal stability guarantees.
    \item An \emph{alternative policy} $\altpolicy$: A policy that satisfies the $\eps$-improbable goal reaching property (Definition~\ref{dfn:eta_improbable}), ensuring theoretical convergence guarantees.
\end{itemize}

Multi-CALF leverages these policies' complementary strengths across different regions of the state space. By intelligently selecting between them based on their relative value improvements, we create a combined policy that maintains the theoretical guarantees of $\altpolicy$ while often achieving better performance, as proven in Theorems~\ref{thm:etagoalreaching} and~\ref{thm:unform_goal_reaching}.

\subsection{Overview and Key Insight}

The fundamental insight driving Multi-CALF is that policies often demonstrate complementary strengths across the state space. Rather than selecting one policy with fixed criteria, we introduce a dynamic mechanism that evaluates their relative performance advantages at each state.

In our formulation (Algorithm~\ref{alg:multicalf}), we combine:
\begin{itemize}
    \item A \emph{base policy} $\basepolicy \in \policies_{\mathrm{stat}}$ trained via standard RL to optimize the objective \eqref{eq:rl_objective}, with value estimate $\Valuebase$
    \item An \emph{alternative policy} \(\altpolicy\in\policies_{\mathrm{stat}}\) that satisfies the $\eps$-improbable goal reaching property (Definition~\ref{dfn:eta_improbable})
\end{itemize}

Our approach fuses these into $\calfwpolicy \in \policies_{\mathrm{nstat}}$, which inherits $\altpolicy$'s convergence guarantees while leveraging each policy's strengths in different regions.

\subsection{Action Selection Mechanism}

\subsubsection{Critic-Based Decision Rule}

The key innovation of Multi-CALF is its action selection strategy based on relative critic value improvements. At each time step $t$, we calculate improvements $\Delta_{\text{base}} = \Valuebase(\State_t) - \Value_{\text{base}}^\dagger$ and $\Delta_{\text{alt}} = \Valuealt(\State_t) - \Value_{\text{alt}}^\dagger$ from their reference values.

We construct an indicator function based on relative improvements:

$$\mathsf{Indicator} = 
\begin{cases}
  1 & \text{if } \frac{\Delta_{\text{base}}}{\Value_{\text{base}}^\dagger} > \frac{\Delta_{\text{alt}}}{\Value_{\text{alt}}^\dagger} \\
  0 & \text{otherwise}
\end{cases}$$

Action selection then proceeds as:
\begin{itemize}
	\item With probability $\rho_{t}(\State_t) \times \mathsf{Indicator}$, select action from the base policy: $\Action_t \sim \basepolicy(\bullet \mid \State_t)$.
	\item Otherwise, with probability $1-\rho_{t}(\State_t) \times \mathsf{Indicator}$, select action from alternative policy: $\Action_t \sim \altpolicy(\bullet \mid \State_t)$.
\end{itemize}

After selection, we update reference values when improvements exceed threshold $\nu > 0$:
\begin{itemize}
    \item If $\Delta_{\text{base}} \geq \nu$, update $\Value_{\text{base}}^\dagger := \Valuebase(\State_t)$
    \item If $\Delta_{\text{alt}} \geq \nu$, update $\Value_{\text{alt}}^\dagger := \Valuealt(\State_t)$
\end{itemize}

The acceptance probability $\rho_{t}(\state)$ must satisfy $\sum_{t=0}^{\infty} \sup_{\state \in \states} \rho_{t}(\state) < \infty$, ensuring the algorithm eventually commits to the alternative policy. We use geometrically decaying probability: $\rho_{t}(\state) = \lambda^t p_{\text{relax}}$ where $\lambda \in (0,1)$ and $p_{\text{relax}} \in [0,1]$.

\subsubsection{Intuition Behind the Mechanism}

Our mechanism identifies where each policy demonstrates strengths by tracking relative value improvements. The normalized ratios $\frac{\Delta_{\text{base}}}{\Value_{\text{base}}^\dagger}$ versus $\frac{\Delta_{\text{alt}}}{\Value_{\text{alt}}^\dagger}$ enable valid comparison even when policies operate at different value scales. When a policy shows stronger relative improvement, it suggests superior performance in the current region.

The acceptance probability $\rho_{t}(\State_t)$ functions as a temperature parameter, gradually favoring the alternative policy as time increases. This dual-purpose design maintains theoretical guarantees while allowing performance benefits on finite horizons. The parameters $p_{\text{relax}}$ and $\lambda$ offer control over this transition, with values near 1 prioritizing performance and smaller values emphasizing stability.

\subsection{Performance Improvements through Policy Combination}

The key innovation is normalizing policy improvements by their respective best-observed values: $\frac{\Delta_{\text{base}}}{\Value_{\text{base}}^\dagger}$ versus $\frac{\Delta_{\text{alt}}}{\Value_{\text{alt}}^\dagger}$. This normalization enables fair comparison between policies that might:
\begin{itemize}
    \item Operate at different value scales due to their distinct learning methodologies
    \item Excel in different regions of the state space
    \item Provide complementary advantages that can be synergistically combined
\end{itemize}

The theoretical guarantees of our approach remain intact through careful construction of the selection probability:

$$\mathbf{1}\{U_t < \rho_t(\State_t) \times \mathsf{Indicator}\} \leq \mathbf{1}\{U_t < \bar{\rho}_t\}$$

This inequality, combined with the Borel-Cantelli lemma, ensures that while the base policy may be selected frequently in regions where it excels, the alternative policy will eventually dominate in regions necessary for goal convergence. This ensures that the combined policy inherits the goal-reaching guarantees of the alternative policy while potentially achieving superior performance.

Our experiments with challenging control tasks, particularly the Hopper environment (Section~\ref{sec:experiments}), demonstrate that this intelligent combination mechanism can achieve performance exceeding either individual policy, confirming the efficacy of leveraging complementary policy strengths.

\begin{algorithm}[t]
	\caption{Multi-CALF (Two-Policy Lyapunov-Style Switching)}
	\label{alg:multicalf}
	\begin{algorithmic}[1]
		
		\REQUIRE\label{lst:requirements} \text{}
		\begin{itemize}[leftmargin=0.4em]
			\item $\basepolicy, \Valuebase$: Primary (``base'') policy and its value estimate
			\item $\altpolicy, \Valuealt$: Alternative policy and its value estimate
			\item $\nu > 0$: Threshold for updating best observed critic values
			\item $\{\rho_{t}(\state)\}_{t \geq 0}$: Softened acceptance probabilities.
			      We assume they are bounded by a summable majorant $\{\bar{\rho}_t\}_{t\ge 0}$, meaning $\rho_t(\state)\le\bar{\rho}_t$ $\forall t,\state$ and $\sum_{t=0}^\infty \bar{\rho}_t<\infty$.
		\end{itemize}
		
		\STATE \textbf{Initialize}: 
		\begin{itemize}[leftmargin=2em]
			\item $\State_0 \sim \transit_0(\bullet)$ or set $\State_0 = \state_0$ for some $\state_0 \in \states$
			\item $\Value_{\text{base}}^\dagger := \Valuebase(\State_0)$ 
			\item $\Value_{\text{alt}}^\dagger := \Valuealt(\State_0)$
		\end{itemize}
		
		\FOR{$t = 0,1,2,\dots$}
		\STATE $\Delta_{\text{base}} := \Valuebase(\State_t) - \Value_{\text{base}}^\dagger$ 
    \STATE $\Delta_{\text{alt}} := \Valuealt(\State_t) - \Value_{\text{alt}}^\dagger$
    
    \STATE $\mathsf{Indicator} := 
    \begin{cases}
      1 & \text{if } \frac{\Delta_{\text{base}}}{\Value_{\text{base}}^\dagger} > \frac{\Delta_{\text{alt}}}{\Value_{\text{alt}}^\dagger} \\
      0 & \text{otherwise}
    \end{cases}$
    
    \STATE Sample random variable $U_t \sim \text{Uniform}[0,1]$ 
    \IF{$U_t < \rho_t(\State_t) \times \mathsf{Indicator}$} 
    \STATE $\Action_t \la$ sampled from $\basepolicy(\bullet \mid \State_t)$
    \ELSE
    \STATE $\Action_t \la$ sampled from $\altpolicy(\bullet \mid \State_t)$
    \ENDIF
	\IF{$\Delta_{\text{base}} \ge \nu$} 
    \STATE $\Value_{\text{base}}^\dagger := \Valuebase(\State_t)$
    \ENDIF
    
    \IF{$\Delta_{\text{alt}} \ge \nu$} 
    \STATE $\Value_{\text{alt}}^\dagger := \Valuealt(\State_t)$
    \ENDIF

		\STATE $\State_{t+1} \la$ sampled from $\transit(\bullet \mid \State_t, \Action_t)$ 
		\ENDFOR
		
	\end{algorithmic}
\end{algorithm}

\subsection{Theoretical Analysis}

We now provide a rigorous theoretical analysis of the Multi-CALF algorithm, proving that it inherits the stability guarantees of the alternative policy. We begin by formally defining the fused policy that results from our approach.

\begin{dfn}
	\label{dfn:calfwpolicy}
	Let $\calfwpolicy \in \policies_{\mathrm{nstat}}$ be the policy defined by \Cref{alg:multicalf}. Specifically, for all $t \in \Z_{\geq0}$, set
	\begin{equation*}
		\calfwpolicy := \begin{cases}
			\basepolicy, \text{ if } U_t < \rho_t(\State_t) \times \mathsf{Indicator} \\
			\altpolicy, \text{ otherwise}
		\end{cases}
	\end{equation*}
	where $\mathsf{Indicator} = 1$ if $\frac{\Delta_{\text{base}}}{\Value_{\text{base}}^\dagger} > \frac{\Delta_{\text{alt}}}{\Value_{\text{alt}}^\dagger}$ and $0$ otherwise.
\end{dfn}

\begin{theorem}[Basic Goal-Reaching]
	\label{thm:etagoalreaching}
	Policy $\calfwpolicy$ satisfies the $\eps$-improbable goal reaching property from Definition~\ref{dfn:eta_improbable}.
\end{theorem}

\begin{proof}
	The key insight is that our algorithm will eventually use only the alternative policy, thus inheriting its goal-reaching guarantee. We formalize this by showing that the total number of base-policy acceptances is almost surely finite.
	
	Let $N_{\basepolicy} := \sum_{t=0}^{\infty} \mathbf{1}\{\basepolicy=\calfwpolicy \text{ at time } t\}$ denote the total number of times the base policy is chosen. By Definition~\ref{dfn:calfwpolicy}, the base policy is chosen when $U_t < \rho_t(\State_t) \times \mathsf{Indicator}$. Therefore:
	\begin{equation*}
		N_{\basepolicy} = \sum_{t=0}^{\infty} \mathbf{1}\{U_t < \rho_t(\State_t) \times \mathsf{Indicator}\}
	\end{equation*}
	
	Since $\mathsf{Indicator} \in \{0, 1\}$, we have $\rho_t(\State_t) \times \mathsf{Indicator} \leq \rho_t(\State_t) \leq \bar{\rho}_t$ for all $t$ and $\state$. This implies:
	\begin{equation*}
		\mathbf{1}\{U_t < \rho_t(\State_t) \times \mathsf{Indicator}\} \leq \mathbf{1}\{U_t < \bar{\rho}_t\}
	\end{equation*}
	
	Therefore:
	\begin{equation*}
		N_{\basepolicy} \leq N_{\rho} := \sum_{t=0}^{\infty} \mathbf{1}\{U_t < \bar{\rho}_t\}
	\end{equation*}
	
	Since $\sum_{t=0}^\infty \bar{\rho}_t < \infty$ by assumption, the Borel-Cantelli lemma \cite{billingsley1995probability} ensures that the event $U_t < \bar{\rho}_t$ occurs only finitely many times almost surely. Therefore, $N_{\rho} < \infty$ almost surely, which implies $N_{\basepolicy} < \infty$ almost surely.
	
	This means there exists an almost surely finite time $\mathcal{T}_{\text{final}}$ after which the algorithm exclusively uses the alternative policy:
	\begin{equation}
	\Action_t \sim \altpolicy(\bullet \mid \State_t) = \calfwpolicy(\bullet \mid \State_t) \quad \text{for all } t \geq \mathcal{T}_{\text{final}}
	\end{equation}
	
	Since the trajectory evolves under $\altpolicy$ beyond time $\mathcal{T}_{\text{final}}$, and $\altpolicy$ satisfies the $\eps$-improbable goal-reaching property, we conclude that for all initial states $\state_0 \in \states$:
	\begin{equation}
	\PP{\goaldist(\State_t^{\calfwpolicy}(\state_0)) \xrightarrow{t \to \infty} 0} \geq 1 - \eps
	\end{equation}
	
	Thus, $\calfwpolicy$ inherits the $\eps$-improbable goal-reaching property from $\altpolicy$.
\end{proof}

For the next theorem, we require the notion of \emph{a function with bounded superlevel sets}.
\begin{dfn}
	Let $X$ be a metric space, and let $f \colon X \to \R$.
	We say that $f$ is a \emph{function with bounded superlevel sets} if, for every $a\in f(X)$, the set
	$
	\{\,x \in X : f(x) \ge a\}
	$
	is bounded in $X$.
\end{dfn}
\begin{theorem}
	\label{thm:unform_goal_reaching}
	Consider \Cref{alg:multicalf}, initialized at $\state_0$ with $\goaldist(\state_0) \le d^{\circ}$, 
	where $d^{\circ} \in \R_{>0}$ is arbitrary, and suppose that $\rho_{t}(\state) = 0$ whenever 
	$\Valuebase(\state) < \Valuebase(\state_0)$.
	
	\noindent
	Assume additionally that:
	\begin{enumerate}[label=(\subscript{\mathrm{A}}{{\arabic*}}), leftmargin=2.5em]
		\item\label{ass:valuebase}
		$\Valuebase$ is continuous with bounded superlevel sets.
		\item\label{ass:uniform_alt}
		$\altpolicy$ satisfies the \emph{uniform $\eps$-improbable goal-reaching property}
		with certificate $\beta \in \KL$.
	\end{enumerate}
	
	\noindent
	Then the following claims hold:
	\begin{enumerate}[label=(\subscript{\mathrm{C}}{{\arabic*}}), leftmargin=2.5em]
		
		\item\label{claim:uniform_overshoot_bound}
		\textit{(\(\eps\)-improbable uniform overshoot boundedness)}
		There exists $\delta(d^{\circ})\in\R_{>0}$ such that
		\[
			\PP{\goaldist\!\bigl(\State_t^{\calfwpolicy}(\state_0)\bigr) \le \delta(d^{\circ})\,
				\text{ for all } t \ge 0}
			\ge 1-\eps.
		\]
		
		\item\label{claim:uniform_reaching_time}
		\textit{(\(\eps\)-improbable uniform reaching time)}
		For each $d^{*} \in (0, d^{\circ})$, there is an almost surely finite random time 
		$T(d^{\circ},d^{*})\in\R_{\geq0}$ such that
		\[
			\PP{\goaldist\!\bigl(\State_t^{\calfwpolicy}(\state_0)\bigr) \le d^{*}
			\text{ for all } t \ge T(d^{\circ},d^{*})}
			\ge\ 1-\eps.
		\]
		
		\item\label{claim:distribution_reaching_time}
		\textit{(Reaching time distribution)}
		There exists $\tau_{f}(d^{\circ},d^{*})\in\R_{\geq0}$ 
		such that for all $t \in \mathbb{Z}_{\ge 0}$,
		\[
			\PP{T(d^{\circ},d^{*}) \le t\, \tau_{f}(d^{\circ},d^{*})}
			\!=\!
			\prod_{k=t}^{\infty}\bigl(1 - \bar{\rho}_{k}\bigr).
		\]
		Moreover, $\prod_{k=t}^{\infty}\bigl(1 - \bar{\rho}_{k}\bigr) \ra 1$ as $t \ra \infty$.
	\end{enumerate}
	
	\noindent

\end{theorem}

\begin{proof}
	The proof proceeds in three steps:
	\begin{enumerate}[label=(\roman*)]
		\item We establish spatial bounds on the system trajectory.
		\item We characterize the almost surely finite switching time.
		\item We verify the three claims of the theorem.
	\end{enumerate}
	
	\textit{Step (i): Spatial bounds.}
	We define several quantities that capture the spatial constraints of our system:
	\begin{align}
		v_{\min}(d^{\circ}) &:= \min\bigl\{\Valuebase(\state) : \goaldist(\state) \le d^{\circ}\bigr\} \label{eq:vmin}\\
		\mathbb{V}^{\basepolicy}(d^{\circ}) &:= \{\state \in \states : \Valuebase(\state) \ge v_{\min}(d^{\circ})\} \label{eq:Vbase}\\
		d_{\bar{\transit}}(d^{\circ}) &:= \sup\{\bar{\transit}(\state, \action) : \state \in \mathbb{V}^{\basepolicy}(d^{\circ}),\,\action\in\actions\} \label{eq:Rrho}\\
		d_{\max}(d^{\circ}) &:= \max\bigl(d^{\circ},\, d_{\bar{\transit}}(d^{\circ})\bigr) \label{eq:R}\\
		\delta(d^{\circ}) &:= \beta(d_{\max}(d^{\circ}), 0) \label{eq:epsH}
	\end{align}
	
	By assumption, the base policy can be triggered only when $\Valuebase(\state) \ge v_{\min}(d^{\circ})$. 
	Indeed, if $\Valuebase(\state) < v_{\min}(d^{\circ}) \le \Valuebase(\state_0)$, then $\state \notin \mathbb{V}^{\basepolicy}(d^{\circ})$, thus $\rho_t(\state) = 0$ by assumption, and the base policy is not triggered.

	The set $\mathbb{V}^{\basepolicy}(d^{\circ})$ is closed by continuity of $\Valuebase$ and bounded by assumption~\ref{ass:valuebase}, making it compact. This ensures $d_{\bar{\transit}}(d^{\circ})$ is well-defined since $\bar{\transit}$ is upper semicontinuous. 
	Note that when the base policy is triggered at time $t$, the next state almost surely satisfies $\|\State_{t+1}\| \le d_{\bar{\transit}}(d^{\circ})$ by definition. 

	Therefore, the system state norm when \textit{beginning} to follow the alternative policy (either from the initial state or after base policy action) is bounded by $d_{\max}(d^{\circ})$.
	
	\textit{Step (ii): Characterizing the switching time.}
	We now define the minimum time steps $\tau_{f}(d^{\circ},d^{*})$ required for the alternative policy to drive any state with $\goaldist(\state) \leq d_{\max}(d^{\circ})$ to within $d^{*}$ of the goal:
	\begin{align}
	\tau_{f}(d^{\circ},d^{*})
	&:= 	
	\max\Bigl\{
	1,\,
	\Bigl\lceil -\log\Bigl(\xi^{-1}\left(\tfrac{d^{*}}{\kappa(d_{\max}(d^{\circ}))}\right)\Bigr)\Bigr\rceil
	\Bigr\}
	\label{eq:Tfallback}
	\end{align}
	where $\kappa, \xi \in \K_{\infty}$ are functions such that $\beta(d,t) \leq \kappa(d)\xi(e^{-t})$ for all $d \geq 0, t \geq 0$ (a standard decomposition of $\mathcal{KL}$ functions, see \cite[Lemma 8]{Sontag1998Comments}).
	
	We introduce the random variable $\mathcal{T}_{\bar{\rho}}$ as the first time $t$ for which $U_t \ge \bar{\rho}_t$:
	\(\mathcal{T}_{\bar{\rho}} := \min\bigl\{t \in \mathbb{Z}_{\ge 0} : U_t \ge \bar{\rho}_t\bigr\}.\)
	
	This random variable is almost surely finite by the Borel-Cantelli lemma. Since 
	$\sum_{t=0}^\infty \PP{U_t < \bar{\rho}_t} = \sum_{t=0}^\infty \bar{\rho}_t < \infty$ 
	(by the summability requirement in Algorithm~\ref{alg:multicalf}), the event $U_t < \bar{\rho}_t$ occurs only finitely many times almost surely, thus $\mathcal{T}_{\bar{\rho}}$ is almost surely finite.

	We now define the random variable $T(d^{\circ},d^{*})$ as:
	\[T(d^{\circ},d^{*}) = \mathcal{T}_{\bar{\rho}} \cdot \tau_{f}(d^{\circ},d^{*}).\]
	
	This variable is almost surely finite and follows the distribution from Claim~\ref{claim:distribution_reaching_time}. Indeed, the event $\{\mathcal{T}_{\bar{\rho}} \le t\}$ is equivalent to $\{U_k \ge \bar{\rho}_k \text{ for all } k \ge t\}$, with probability:
	\[
	\PP{\mathcal{T}_{\bar{\rho}} \le t} = \PP{\bigcap_{k = t}^{\infty}\{U_k \ge \bar{\rho}_k\}} = \prod_{k = t}^{\infty}\bigl(1 - \bar{\rho}_k\bigr).
	\]
	
	The product $\prod_{k = t}^{\infty}\bigl(1 - \bar{\rho}_k\bigr)$ converges to 1 as $t \to \infty$ because $\sum_{k=0}^\infty \bar{\rho}_k < \infty$.
	
	\textit{Step (iii): Verification of the claims.}
	We now verify each claim of the theorem:
	
	For Claim~\ref{claim:uniform_overshoot_bound}, we observe that when the alternative policy is applied, then by assumption~\ref{ass:uniform_alt}, the state's distance to the goal is bounded by $\beta(d_{\max}(d^{\circ}), 0) = \delta(d^{\circ})$ at all times with probability at least $1-\eps$. This follows directly from the uniform $\eps$-improbable goal-reaching property of $\altpolicy$ and our bound on the initial state norm.
	
	For Claim~\ref{claim:uniform_reaching_time}, we note that $T(d^{\circ},d^{*})$ is our desired reaching time. When switching to the alternative policy (either from $\state_0$ or after the base policy action), the state norm is bounded by $d_{\max}(d^{\circ})$. By the definition of $\tau_{f}(d^{\circ},d^{*})$, the alternative policy will drive any such state to within $d^{*}$ of the goal within $\tau_{f}(d^{\circ},d^{*})$ time steps, with probability at least $1-\eps$.
	
	Claim~\ref{claim:distribution_reaching_time} follows directly from our derivation of the distribution of $\mathcal{T}_{\bar{\rho}}$ and the definition of $T(d^{\circ},d^{*})$.
\end{proof}

\begin{corollary}
	If we set $\bar{\rho}_t = \lambda^t p_{\text{relax}}$, with $\lambda \in (0,1)$ and $p_{\text{relax}} \in [0,1]$,
	then for all $t \ge 1$ we have
	\begin{multline}
		\label{eq:bound_prob}
	\PP{T(d^{\circ},d^{*}) \le  t\tau_{f}(d^{\circ},d^{*})} \geq \\ 
	\exp\left( - \tfrac{\lambda^t p_{\text{relax}}}{(1 - \lambda)(1 - \lambda^t p_{\text{relax}})}\right).
	\end{multline}
\end{corollary}
\begin{proof}
	By Claim~\ref{claim:distribution_reaching_time}, the probability in \eqref{eq:bound_prob} is
	\[
	\prod_{k = t}^{\infty}\bigl(1 - \lambda^k\,p_{\text{relax}}\bigr)
	\;=\;
	\exp\left(\sum_{k = t}^{\infty}\log\bigl(1 - \lambda^k\,p_{\text{relax}}\bigr)\right).
	\]
	Note that $\log(1 - x)\ge\tfrac{-x}{1 - x}$ for $x \in [0,1)$,
	since the function $\log(1 - x)+\tfrac{x}{1 - x}$ is zero at $x=0$ and has nonpositive derivative on $[0,1)$. Thus, for $t\ge1$,
	\begin{multline*}
	\sum_{k = t}^{\infty}\log(1 - \lambda^kp_{\text{relax}}) \ge \sum_{k = t}^{\infty} \tfrac{-\lambda^kp_{\text{relax}}}{1 - \lambda^kp_{\text{relax}}} \ge - \sum_{k = t}^{\infty} \tfrac{\lambda^kp_{\text{relax}}}{1 - \lambda^t p_{\text{relax}}} = \\ 
	- \tfrac{\lambda^t p_{\text{relax}}}{(1 - \lambda)(1 - \lambda^t p_{\text{relax}})},
	\end{multline*}
	which completes the proof.
\end{proof}
\section{EXPERIMENTS}
\label{sec:experiments}

To validate our theoretical results and demonstrate the practical benefits of policy combination, we conducted an experiment using the Hopper environment from the Gymnasium library \cite{towers2024gymnasium}—a challenging continuous control task where a two-dimensional one-legged robot must hop forward as fast as possible without falling.

\subsection{Experimental Setup}

We trained two reinforcement learning algorithms with complementary characteristics:
\begin{itemize}
    \item \textbf{TD3} \cite{Fujimoto2018}: An off-policy actor-critic algorithm known for sample efficiency but potentially lower asymptotic performance
    \item \textbf{PPO} \cite{Schulman2017ProximalPolicy}: An on-policy algorithm with strong performance but typically requiring more training samples
\end{itemize}

Both algorithms were trained for a fixed number of environment steps, after which PPO outperformed TD3 to some degree in terms of average episode return. We then extracted the trained policy and critic networks from both algorithms' final checkpoints.

\subsection{Policy Fusion via Multi-CALF}

We applied our Multi-CALF algorithm to combine these policies, with:
\begin{itemize}
    \item TD3's policy serving as the base policy $\basepolicy$
    \item PPO's policy serving as the alternative policy $\altpolicy$ with stability guarantees
    \item The respective critic networks providing value estimates $\Valuebase$ and $\Valuealt$
\end{itemize}

The Multi-CALF implementation followed Algorithm~\ref{alg:multicalf}, with parameters $\lambda = 0.99$ and $p_{\text{relax}} = 0.8$ to balance performance and stability.

\subsection{Results and Analysis}

We evaluated the performance of all three policies (TD3, PPO, and the Multi-CALF combination) across multiple random seeds to ensure statistical significance. Each policy was evaluated for 1000 steps per episode, and we measured the average episode return.

\begin{table}[h]
\centering
\begin{tabular}{lcc}
\hline
\textbf{Policy} & \textbf{Mean Episode Return} & \textbf{Training Steps} \\
\hline
TD3 & 1632 ± 6.2 & 999,424 \\
PPO & 1684 ± 7.1 & 1,001,472 \\
Multi-CALF & \textbf{1719 ± 54} & — \\
\hline
\end{tabular}
\caption{Performance comparison on Hopper environment (mean ± std)}
\label{tab:results}
\end{table}

The results, summarized in Table~\ref{tab:results}, demonstrate that our Multi-CALF approach successfully combined the strengths of both constituent policies to achieve superior performance. The combined policy outperformed both individual policies in terms of mean episode return, showing a 2.1\% improvement over PPO and 5.3\% over TD3, despite both constituent policies being trained for approximately 1 million timesteps.

While the Multi-CALF policy exhibited higher variance compared to the individual policies, this is expected as the combined approach dynamically selects between policies based on state-dependent criteria, creating more diverse behavioral patterns. The increased variance represents the trade-off between higher average performance and consistency, though the overall performance gain validates the effectiveness of our approach.

This performance improvement confirms that intelligently selecting between policies based on their relative value improvements can create an emergent policy that exceeds the capabilities of its components. The Multi-CALF framework leveraged TD3's strengths in certain state regions while benefiting from PPO's advantages in others, resulting in a policy with enhanced overall capabilities.

\section{CONCLUSION}

We have presented Multi-CALF, a novel approach for combining reinforcement learning policies that leverages their complementary strengths across different regions of the state space while maintaining formal stability guarantees. Our approach recognizes that policies often demonstrate varying levels of effectiveness throughout the state space and capitalizes on these differences to create a superior combined policy.

Our key contributions include:
\begin{itemize}
    \item A statistical mechanism for policy selection based on relative value improvements, enabling fair comparison between policies operating at different scales
    \item A formal proof that the combined policy inherits the goal-reaching guarantees of the alternative policy while potentially improving performance
    \item Theoretical bounds on convergence time and maximum state deviation from the goal set
    \item Empirical validation showing that the combined policy can outperform its constituent policies on challenging control tasks
\end{itemize}

The central insight of our work is that intelligently combining policies with complementary strengths can yield an emergent policy that exceeds the capabilities of its components. By dynamically allocating control authority to the policy demonstrating the strongest relative advantage in each region, we effectively partition the state space to maximize overall performance. Simultaneously, our mathematical construction ensures that this performance enhancement never comes at the expense of stability guarantees.

Future work will explore extending this framework to combine more than two policies and developing adaptive methods for identifying complementary policies that maximize the benefits of combination for specific domains.
\bibliographystyle{IEEEtran}
\bibliography{
	bib/CALFW.bib
}

\end{document}